\documentclass{article}
\PassOptionsToPackage{table}{xcolor}
\usepackage{iclr2027_conference,times}
\usepackage[OT1]{fontenc}
\usepackage{iftex}
\ifXeTeX
  \usepackage{newunicodechar}
  \newunicodechar{’}{'}
  \newunicodechar{‘}{`}
  \newunicodechar{“}{``}
  \newunicodechar{”}{''}
  \newunicodechar{–}{--}
  \newunicodechar{—}{---}
\fi

\usepackage{amsmath,amsfonts,bm}

\def\eqref#1{equation~\ref{#1}}

\def\1{\bm{1}}

\DeclareMathAlphabet{\mathsfit}{\encodingdefault}{\sfdefault}{m}{sl}
\SetMathAlphabet{\mathsfit}{bold}{\encodingdefault}{\sfdefault}{bx}{n}

\newcommand{\E}{\mathbb{E}}

\newcommand{\softmax}{\mathrm{softmax}}

\usepackage{amssymb}
\usepackage[xcharter,bigdelims,vvarbb]{newtxmath}
\usepackage[scaled=1.1]{zlmtt}
\usepackage{hyperref}
\usepackage{url}
\usepackage{booktabs,tabularx,graphicx,wrapfig,capt-of,subcaption,pifont,placeins}
\usepackage{algorithm}
\usepackage{algpseudocode}
\algrenewcommand{\algorithmicrequire}{\textbf{Input:}}
\algrenewcommand{\algorithmicensure}{\textbf{Output:}}
\definecolor{tableBlue}{HTML}{6089B1}
\definecolor{tableBestPale}{HTML}{DFEDF7}
\definecolor{tableLowPale}{HTML}{EAF2F9}
\definecolor{tableGroup}{HTML}{ECEFF2}
\definecolor{tableFocus}{HTML}{F1F6FA}
\newcommand{\tablebest}[1]{\cellcolor{tableBestPale}\textbf{#1}}
\newcommand{\tablekey}[2]{\begingroup\setlength{\fboxsep}{1.2pt}%
  \colorbox{#1}{\bfseries #2}\endgroup}

\newtheorem{proposition}{Proposition}
\newcommand{\method}{\textnormal{HiLo-RD}}
\newcommand{\mask}{\texttt{[MASK]}}

\title{HiLo-RD: Cross-Budget Self-Distillation for\\ Masked Diffusion Language Models}
\author{\begin{minipage}{\dimexpr\textwidth-2\tabcolsep\relax}\centering
Xiang Li\textsuperscript{1}\quad
Chengting Yu\textsuperscript{2}\quad
Jungang Li\textsuperscript{1}\quad
Yilun Zheng\textsuperscript{3}\\[3pt]
Junzhuo Li\textsuperscript{1}\quad
Xingyao Zhang\textsuperscript{4}\quad
Wenbo Su\textsuperscript{4}\quad
Xuming Hu\textsuperscript{1,5}\\[6pt]
\normalfont\small
\textsuperscript{1}The Hong Kong University of Science and Technology (Guangzhou)\\
\textsuperscript{2}Zhejiang University\\
\textsuperscript{3}Nanyang Technological University\\
\textsuperscript{4}Alibaba Group\\
\textsuperscript{5}The Hong Kong University of Science and Technology
\end{minipage}}
\hypersetup{
  pdftitle={HiLo-RD: Cross-Budget Self-Distillation for Masked Diffusion Language Models},
  pdfauthor={Xiang Li, Chengting Yu, Jungang Li, Yilun Zheng, Junzhuo Li, Xingyao Zhang, Wenbo Su, Xuming Hu},
  pdfsubject={Preprint},
  unicode=true
}
\iclrfinalcopy
\begin{document}
\maketitle
\fancyhead{}
\begin{abstract}

We propose \textbf{HiLo-RD} (\textbf{Hi}gh-to-\textbf{Lo}w Budget \textbf{R}epresentation \textbf{D}istillation) to improve masked diffusion language models (MDLMs) under a limited generation budget. Unlike autoregressive models, MDLMs choose both what to generate and where to commit it. This creates two coupled challenges: predicting useful content before reasoning is complete, and deciding when to fix answer tokens that shape later predictions. We identify \textit{\textbf{budget-induced premature commitment}}: even when a correct solution fits, the model can fix answer fragments near the window boundary before resolving intermediate steps. These fragments become context for later predictions and redirect subsequent reasoning. Existing post-training methods do not directly transfer the stronger trajectories that the same model follows with a larger window.
HiLo-RD addresses this gap by (1) pairing large-window teacher states with cropped small-window views and (2) matching their hidden representations at committed positions. This supervision guides both token prediction and commitment order without an external teacher. Across mathematics benchmarks, HiLo-RD improves low-budget accuracy and shortens responses; these gains extend to code generation without code-domain training. On GSM8K, accuracy rises by 9.4 percentage points as mean readable length falls by about 25\%. Attention diagnostics and fixed-budget order interventions support the proposed mechanism. Together, these results show that large-window reasoning can improve small-window decoding.

\end{abstract}
\section{Introduction}

\begin{wrapfigure}[19]{r}{0.46\textwidth}
\vspace{-6pt}
\centering
\captionsetup{font=small}
\begin{subfigure}{\linewidth}
\centering
\includegraphics[width=0.93\linewidth]{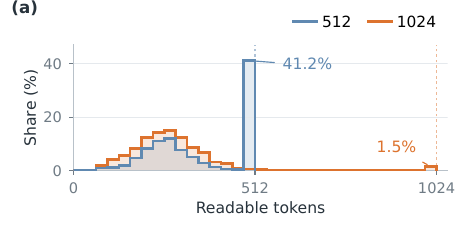}
\end{subfigure}\par\vspace{2pt}
\begin{subfigure}{\linewidth}
\centering
\includegraphics[width=0.93\linewidth]{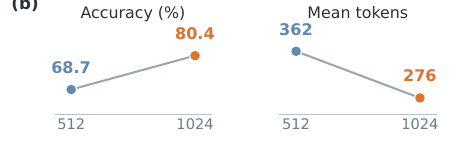}
\end{subfigure}
\caption{\textbf{More room, shorter and more accurate answers.}
At $L=T=512$ (blue) and $L=T=1024$ (orange), (a) shows response-length shares
in 32-token bins on common axes and (b) compares accuracy and mean readable
length. Values are approximate readings from the archived aggregate plot.}
\label{fig:panic-distribution}
\end{wrapfigure}

Masked diffusion language models (MDLMs) offer a route beyond left-to-right generation by reconstructing masked text under bidirectional attention \citep{sahoo2024simple,nie2025large,nie2026improvedlargelanguagediffusion}. In pure diffusion decoding, the model predicts candidates throughout a response window and chooses which positions to commit at each update. This allows reasoning to unfold across denoising steps and potentially across independent subproblems in parallel \citep{svete2026on}. For practical decoding, however, the response window and the number of updates are finite. The model must therefore decide both \emph{what} to write and \emph{when} to fix each part of the answer within a limited budget.

\label{sec:panic-observations}
Consider a mathematical question for which the same model gives different answers under two response-window sizes. With the smaller window, it can fix answer fragments near the end before filling the reasoning in between, even though a correct solution would fit. With a larger window, it may defer the answer and reach a shorter, correct solution. We call this failure \emph{budget-induced premature commitment}. Figure~\ref{fig:panic-distribution} shows the broader aggregate pattern: fewer responses end near the larger window's limit, while accuracy rises and mean readable length falls. This raises our central question: \emph{If a larger window yields shorter, more accurate answers, can we distill that advantage into a smaller one?}

This example exposes two coupled challenges in low-budget diffusion reasoning. First, the model must choose \emph{what} to predict while intermediate steps are still unresolved. Second, it must decide \emph{when} to commit an answer fragment, because committed tokens cannot be revised and become context for every later prediction. A premature answer can thus shape the reasoning subsequently filled around it. Existing inference-time methods adjust the available space or restrict the order of generation \citep{li2025beyondfixed,arriola2025block}; post-training methods learn from demonstrations, outcome rewards, or self-distillation targets \citep{zhao2025d1,zhao2026selfdistilled,dat2026dopsd}. These approaches improve decoding in different ways, but do not directly transfer the same model's more effective large-window computation into its small-window decisions.

We investigate how an early commitment changes the subsequent trajectory. Before the model fixes the response tail, attention concentrates near the small-window boundary. Once tail tokens are committed, unresolved positions attend more to them and less to the established reasoning prefix. This feedback makes later predictions depend on an answer selected before its supporting steps are complete. In a set of 155 selected cases where native low-budget decoding is wrong and native high-budget decoding is correct, temporarily delaying tail commitments at the original low budget recovers correct answers in 43.2\% of cases, compared with 30.3\% for a matched order control. These diagnostics connect the budget, commitment order, and downstream reasoning, and show that a better path can remain available without increasing the deployment budget.

We summarize our main contributions as follows.
\par\noindent\ding{182} We identify \emph{budget-induced premature commitment}: a small window can induce an early answer even when a correct solution fits. Attention diagnostics and fixed-budget interventions trace how such commitments influence subsequent reasoning.
\par\noindent\ding{183} We introduce \textbf{HiLo-RD} (High-to-Low Budget Representation Distillation). It aligns large-window teacher states with their cropped small-window views and matches hidden representations at committed positions, using the model's own trajectories to guide low-budget decoding.
\par\noindent\ding{184} At the same small inference budget, HiLo-RD raises GSM8K accuracy from 68.84\% to 78.24\% while reducing mean readable length from 362.8 to 271.3 tokens. The gains extend to code generation without code-domain training, and successive distillation improves the low-budget quality--length trade-off on a fixed training subset.

Overall, HiLo-RD transfers the same model's larger-window computation into more accurate, more concise small-window responses. The results show that the way a diffusion model commits its answer is a useful target for improving reasoning under a limited generation budget.

\section{Preliminaries and Related Work}
\label{sec:preliminary}

\subsection{Masked Diffusion Decoding}
\label{sec:pure-decoding}

We consider pure diffusion decoding with LLaDA-8B-Instruct \citep{nie2025large}, a masked diffusion language model (MDLM). The full response is decoded as one bidirectional block. Once a token is committed, it cannot be revised.

Write the prompt and response state before each update as
\[
 \mathbf{q}=[q_1,\ldots,q_P],\qquad
 \mathbf{y}_t=[y_{t,1},\ldots,y_{t,L}],\qquad \mathbf{y}_0=\mask^L.
\]
A response position $j$ has input index $P+j$. The masked and observed (committed) positions are
\begin{equation}
 \mathcal M_t=\{j\in[L]:y_{t,j}=\mask\},\qquad
 C_t=[L]\setminus\mathcal M_t,
 \label{eq:decoding-positions}
\end{equation}
where $[L]=\{1,\ldots,L\}$. For vocabulary $\mathcal V$, the model with parameters $\theta$ predicts at each masked position
\[
 \hat y_{t,j}=\arg\max_{v\in\mathcal V}
 p_\theta(v\mid\mathbf{q},\mathbf{y}_t,j).
\]
A confidence rule ranks masked positions, suppressing EOS/EOT candidates in the reference decoder. Selected positions receive their predicted tokens; all other positions remain unchanged. Selection need not follow response order, and the ranking score need not equal token probability.

A budget specifies the response-window length and planned number of updates. A run can terminate after $R$ updates:
\[
 B\triangleq(L,T),\qquad 1\leq R\leq T\leq L.
\]
The full decoder fills the window in $T$ updates. When $L=T$, it commits one position per update; when $T<L$, at least one update commits multiple positions. Enlarging the window changes both the visible context and the positions competing for commitment.

\subsection{Decoding Budgets and Realized Computation}
\label{sec:budget-cost}

At each update, the decoder processes the prompt and full response window afresh. For $N=P+L$ input positions, dense forward cost (including the vocabulary projection) scales as
\begin{equation}
 F(N)=\Theta\!\left(
 n_\ell\left[Nd^2+N d d_{\mathrm{ff}}+N^2d\right]
 +Nd|\mathcal V|
 \right).
 \label{eq:forward-cost}
\end{equation}
Here $n_\ell$ is the number of layers, $d$ the hidden width, and $d_{\mathrm{ff}}$ the feed-forward width. Bidirectional states can change when new response tokens are committed, so this decoder does not reuse a KV cache across updates \citep{vaswani2017attention,nie2025large}. Using the forward-pass cost as a proxy gives
\begin{equation}
 \mathcal C(\mathbf{q},B;R)\triangleq R F(P+L),\qquad
 \mathcal C_{\mathrm{full}}(\mathbf{q},B)\triangleq T F(P+L).
 \label{eq:decoding-cost}
\end{equation}
Each avoided update saves one full-window forward. Section~\ref{sec:fast-rollout} specifies the stopping rules.

\subsection{Related Work}
\label{sec:related-work}

\paragraph{Diffusion language models.}
Masked diffusion models predict missing tokens from visible context, allowing positions to be committed out of order \citep{sahoo2024simple,nie2025large}. LLaDA and iLLaDA train from scratch, while DiffuLLaMA and Dream adapt autoregressive checkpoints \citep{nie2026improvedlargelanguagediffusion,gong2025scaling,ye2025dream}. Block diffusion orders generation between blocks \citep{arriola2025block}; WeDLM uses causal attention for cache-friendly recovery \citep{liu2025wedlm}. Our single-block setting keeps all masked response positions eligible at each update.

\paragraph{Post-training methods.}
Masked SFT learns from partially masked references; Diffu-GRPO uses outcome rewards \citep{nie2025large,zhao2025d1}. OPSD and dOPSD self-distill sampled trajectories \citep{zhao2026selfdistilled,dat2026dopsd}, while T3D distills trajectories to reduce denoising steps \citep{zhang2026fewstepdiffusionlanguagemodels}. \method\ pairs a teacher state with its cropped small-window view and matches their hidden states to improve low-budget reasoning.

\section{Budget-Induced Premature Commitment and Its Consequences}
\label{sec:panic-setup}
\label{sec:panic-analysis}

A response window can fit a correct solution yet still change when a masked diffusion language model fixes its answer. We study 155 paired cases selected for an incorrect answer at $L=T=512$ and a correct answer from the same model at $L=T=1024$. In the smaller window, the model commits an answer fragment near the response tail before the intervening reasoning is complete. We connect this failure to boundary-focused attention, subsequent attention to the committed tail, and recovery when tail commitment is delayed within the original budget.

\subsection{Boundary Attention Precedes Premature Commitment}
\label{sec:panic-budget-attention}

Figure~\ref{fig:panic-budget-sensitivity} examines attention immediately before the first ordinary tail commitment. Under low-budget decoding, a band forms near the response-window boundary across the middle and later layers while intermediate positions remain unresolved. Appending 512 visible \mask{} positions to the same partial state and query weakens this band.

\begin{figure}[!ht]
\centering
\includegraphics[width=\linewidth]{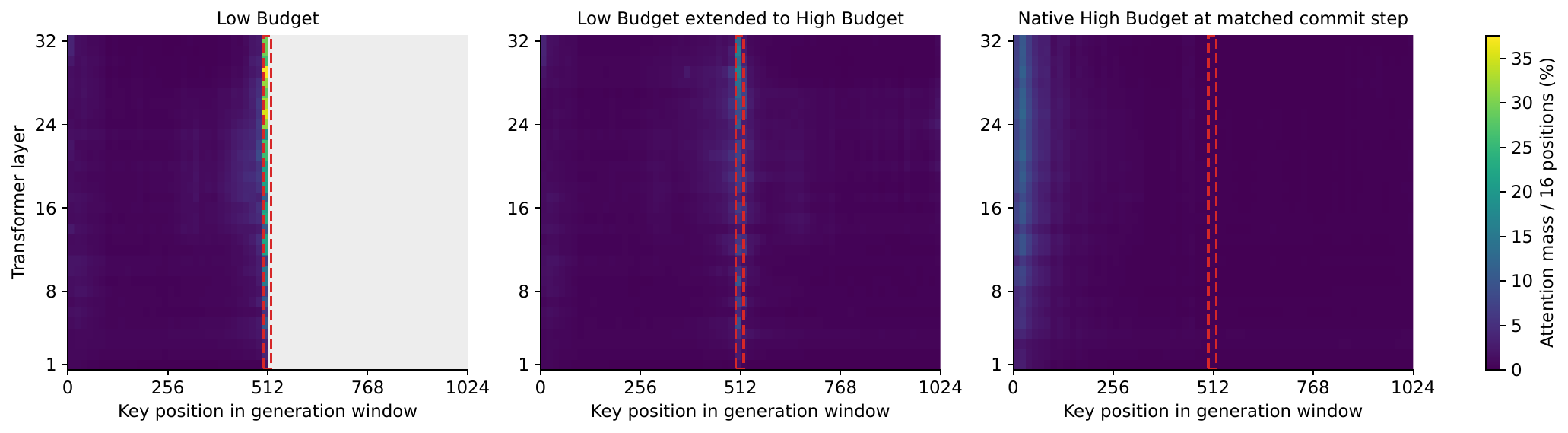}
\caption{\textbf{Boundary-focused attention precedes tail commitment.}
Visible-window expansion relaxes this focus at the same state and query (left vs. middle), while native high-budget decoding attends more to earlier response positions at the matched update (right). Panels share an attention scale; red dashed lines mark the original low-budget boundary.}
\label{fig:panic-budget-sensitivity}
\end{figure}

\FloatBarrier

Native high-budget decoding, followed from initialization to the matched update, instead attends more to earlier decoded reasoning. The within-state comparison isolates the effect of visible response space; the native trajectory shows the computation reached when that space is available throughout decoding. The smaller window changes attention before the first tail token is fixed.

\subsection{Committed Tail Content Attracts Subsequent Attention}
\label{sec:panic-tail-reading}

\begin{wrapfigure}[10]{r}{0.44\textwidth}
\centering
\includegraphics[width=\linewidth]{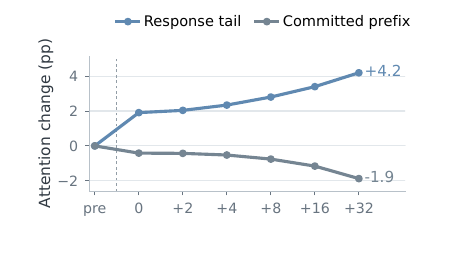}
\caption{\textbf{Attention shifts toward the committed tail.} Changes from the pre-commitment update, averaged across layers and 145 valid traces. The dashed line marks commitment; $0$ is immediately after it.}
\label{fig:panic-tail-commitment}
\vspace{-3pt}
\end{wrapfigure}

Once the tail is committed, it becomes part of the bidirectional context for unresolved positions.

Figure~\ref{fig:panic-tail-commitment} follows the same unfinished front positions across 145 valid traces. Immediately after the first tail commitment, their mean attention to the tail rises by 1.91 percentage points, while attention to the established prefix falls by 0.41 points. Along the subsequent decoding trajectory, these changes reach $+4.19$ and $-1.87$ points at $+32$ updates. The early answer fragment thus enters the context used to complete the still-masked reasoning steps.

\subsection{Tail Delay Restores Correct Answers}
\label{sec:panic-trajectory-interventions}

\begin{wrapfigure}{r}{0.50\textwidth}
\vspace{-6pt}
\centering
\includegraphics[width=\linewidth]{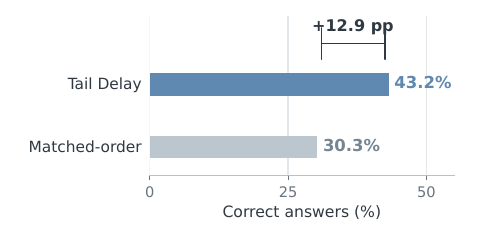}
\caption{\textbf{Tail delay improves low-budget recovery.} Correct continuations in 155 cases selected for an incorrect native 512-token answer and a correct native 1024-token answer. Both interventions use $L=T=512$.}
\label{fig:panic-decoding-intervention}
\end{wrapfigure}

A fixed-budget intervention tests whether postponing that decision opens a better continuation. \emph{Tail Delay} excludes the original response-tail positions during the first 10\% of remaining updates. A \emph{Matched-order Control} perturbs selection at the same update indices but excludes its own current highest-scoring position, without a tail-specific rule. Both branches start from the same pre-event state and retain the $L=T=512$ window, predictor, and decoding schedule.

Tail Delay reaches the correct answer in 67 of the 155 selected cases (43.2\%), compared with 47 (30.3\%) for the order control, a paired gain of 12.9 percentage points (95\% bootstrap interval: 7.7--18.7). The advantage over a matched order perturbation supports postponing the tail as the useful change. Successful paths are accessible within the smaller window. This motivates HiLo-RD to transfer the high-budget model's commitment-aware representations into low-budget decoding.

\section{\method: Cross-Budget Representation Self-Distillation}
\label{sec:method}

\subsection{Cross-Budget Representation Distillation}
\label{sec:views}
\label{sec:loss}

\method\ transfers the commitment-aware representations of the model's own high-budget trajectories into low-budget decoding. A frozen teacher and a trainable student start from the same checkpoint and receive the same prompt $q$. Their budgets are $B_T=(L_T,T_T)$ and $B_S=(L_S,T_S)$, where $L_S<L_T$. The high-budget teacher follows a trajectory that is less affected by Budget-Induced Premature Commitment; its states therefore encode more reliable decisions about both \emph{what} to predict and \emph{where} to commit.

\begin{figure}[!ht]
\centering
\includegraphics[width=\linewidth]{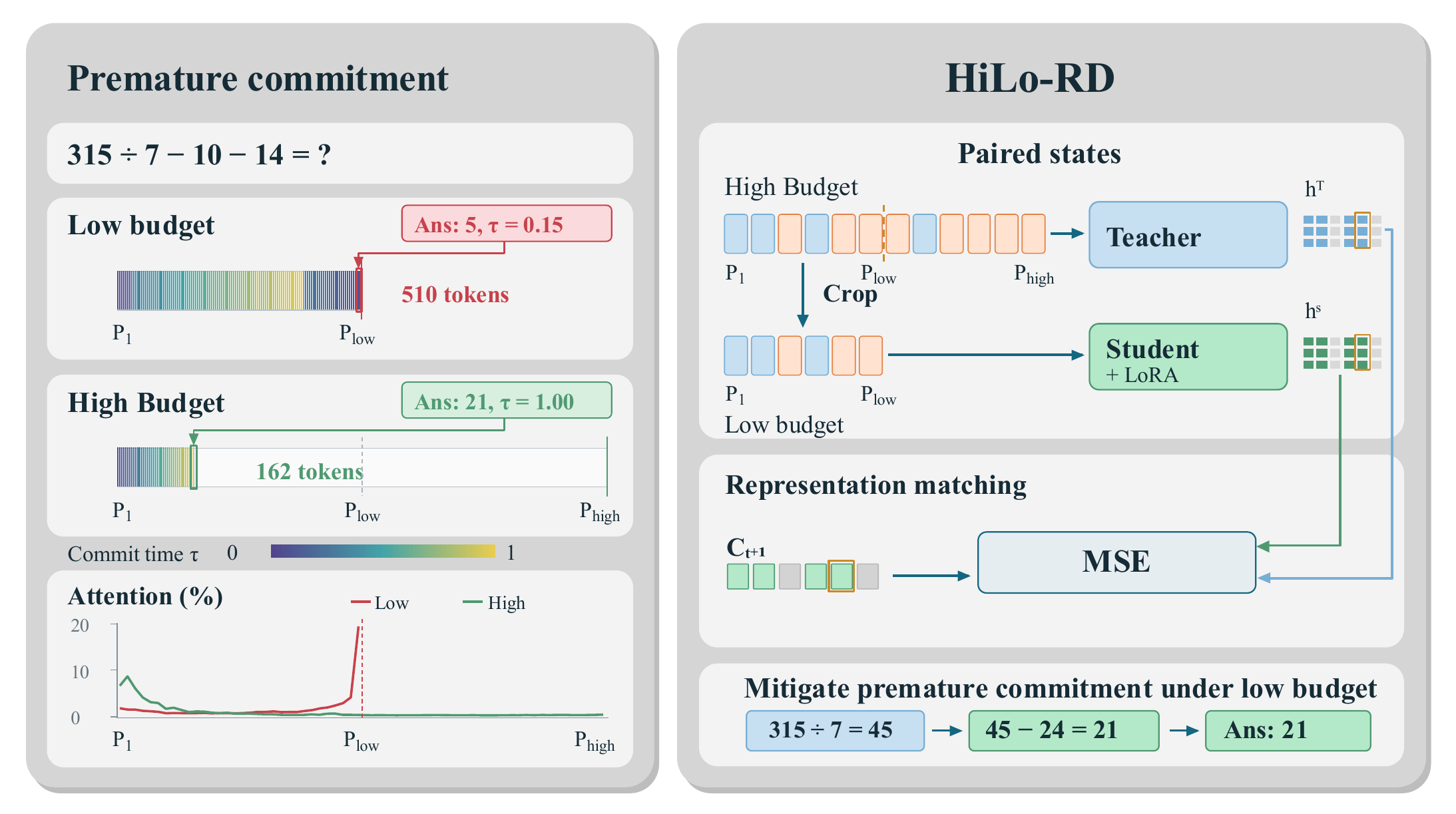}
\caption{\textbf{\method\ overview.} \textbf{Left: Budget-Induced Premature Commitment.} For the same prompt, low-budget decoding concentrates attention near its response boundary $P_{\mathrm{low}}$, fixes an incorrect answer early, and produces a response that fills its allocated window. With a larger budget, attention remains on earlier reasoning positions; the same model commits the correct answer at the final decoding stage and produces a much shorter response. Colors indicate normalized commitment time. \textbf{Right: Cross-budget representation self-distillation.} A frozen teacher processes a high-budget state, while the trainable student receives the crop ending at $P_{\mathrm{low}}$. Hidden-state MSE over the cumulative teacher-committed set $C_{t+1}$, including newly and previously committed positions, distills the reliable high-budget trajectory encoded in the teacher representations into the low-budget student. The trained student follows a coherent low-budget path and reaches the correct answer.}
\label{fig:framework}
\end{figure}

For each retained teacher state $y_t$, the student receives its cropped response state $y_{t,1:L_S}$. The overlapping positions preserve their token values and position indices, yielding an exact cross-budget correspondence. Let $H_\theta$ denote the final normalized hidden states before the shared vocabulary head. At each response position $j\leq L_S$, the paired representations are
\begin{align}
 h^T_{t,j}&=H_{\bar\theta}(q,y_t)_{P+j},
 \label{eq:teacher-hidden}\\[-2pt]
 h^S_{t,j}&=H_\theta(q,y_{t,1:L_S})_{P+j},
 \label{eq:student-hidden}
\end{align}
where $\bar\theta$ is the frozen teacher. The teacher representation is computed over its high-budget trajectory before cropping, whereas the student must produce its representation from the low-budget view.

We supervise the cumulative positions committed by the teacher after update $t$: $m_{t,j}=\mathbf{1}\{j\in C_{t+1}\}$ and $n_t=\sum_{j=1}^{L_S}m_{t,j}$. This includes both previously committed positions and newly selected positions whose inputs in $y_t$ remain masked. For $n_t>0$, HiLo-RD minimizes
\begin{equation}
 \ell_t=\frac{1}{d\,n_t}\sum_{j=1}^{L_S}m_{t,j}
 \left\|h^S_{t,j}-\operatorname{sg}[h^T_{t,j}]\right\|_2^2,
 \label{eq:state-loss}
\end{equation}
where $d$ is the hidden dimension and $\operatorname{sg}$ stops the teacher gradient. We average this loss over valid retained states and contributing trajectories. The high-budget trajectory is encoded directly in each teacher representation $h^T_{t,j}$: it reflects the trajectory's bidirectional context, commitment history, and its reliable decisions about both what to predict and where to commit. Minimizing Eq.~(\ref{eq:state-loss}) forces the low-budget representation $h^S_{t,j}$ to internalize this high-budget trajectory information at the corresponding position. HiLo-RD is therefore a trajectory distillation method at the representation level, rather than a loss on isolated tokens or separately specified commitment labels. The shared vocabulary head maps the resulting representation alignment to locally aligned token distributions; Appendix~\ref{app:proofs} gives the corresponding bound. Since commitment confidence is also derived from these representations, the same objective improves both token prediction and commitment order.

\subsection{Efficient Trajectory Acquisition}
\label{sec:fast-rollout}
Full-completion decoding continues until every allocated position is filled, even when the readable response is already complete. An early EOT alone is insufficient to stop: an EOT token may be committed before earlier reasoning positions have been resolved. We instead require a \emph{clean EOT}: the leftmost committed EOT and every preceding response position must be committed. With unchanged updates and irreversible commitment, further decoding cannot alter the response extracted before that EOT (Proposition~\ref{prop:stop-prefix}).

Teacher acquisition can additionally stop when its first $L_{\mathrm S}$ response positions are filled, since these define the supervision window. This cap preserves the completed prefix but may truncate a longer solution; it is distinct from response-preserving clean-EOT stopping. Before termination, all positions in the teacher's full window remain eligible for commitment. Both rules are applied during rollout, so they avoid forwards rather than merely discarding stored states.

If stopping reduces the realized updates from $T$ to $R$, the cost becomes $R F(P+L)$ instead of $T F(P+L)$ under Section~\ref{sec:budget-cost}. Each forward still processes the full window. Efficient acquisition therefore reduces continuation cost, while distillation improves the quality attainable at a smaller budget. Exact preservation assumes matched decoding and first-EOT extraction; implementation-specific extraction details are given in Appendix~\ref{app:metrics}.

\section{Experiments}
We evaluate \method\ through four research questions covering low-budget reasoning quality, transfer, rollout efficiency, and iterative self-distillation.

\begingroup
\setlength{\parskip}{1pt}
\noindent\textbf{RQ1} (Performance): Can \method\ distill high-budget reasoning into accurate, compact low-budget decoding?\par
\noindent\textbf{RQ2} (Generalization): Do these decoding gains generalize to out-of-domain code generation?\par
\noindent\textbf{RQ3} (Efficiency): Can efficient trajectory acquisition cut pure-diffusion rollout cost without quality loss?\par
\noindent\textbf{RQ4} (Self-Evolving): Can successive self-distillation drive self-evolving quality--budget refinement?\par
\endgroup

\subsection{Experimental Setup}

\paragraph{Data Preparation.} We train separately on the full GSM8K training set \citep{cobbe2021gsm8k} and the original MATH training set \citep{hendrycks2021math}, and evaluate the corresponding checkpoints on the GSM8K test set and MATH500 \citep{lightman2024verify}, respectively. For out-of-domain evaluation, we directly test mathematics-trained checkpoints on HumanEval \citep{chen2021humaneval} and MBPP \citep{austin2021mbpp}, without additional code-domain training.

\paragraph{Baselines.} We compare \method\ with the original LLaDA-8B-Instruct (\emph{vanilla}) \citep{nie2025large} and four post-training methods: masked supervised fine-tuning (Masked SFT), Diffu-GRPO \citep{zhao2025d1}, on-policy self-distillation (OPSD) \citep{zhao2026selfdistilled}, and dOPSD \citep{dat2026dopsd}. All comparisons use pure diffusion decoding. For vanilla and \method, we additionally compare full decoding with the fast-rollout stopping rule in Section~\ref{sec:fast-rollout} to assess its effect on quality and efficiency. Section~\ref{sec:related-work} situates these post-training objectives.

\paragraph{Training and Evaluation.} We perform parameter-efficient adaptation with LoRA \citep{hu2022lora}, optimizing the adapter parameters with AdamW \citep{loshchilov2019adamw} while keeping the backbone weights frozen. The \method\ training runs use eight NVIDIA A100 80GB GPUs. We measure mathematical reasoning performance using accuracy and code-generation performance using pass@$k$ in the main comparison. We also report mean readable token length (RTL) and mean wall-clock time per request to characterize response length and inference cost. Additional training and evaluation details are provided in Appendix~\ref{app:implementation}.

\subsection{Low-Budget Performance and Cross-Domain Transfer (RQ1 \& RQ2)}
\label{sec:performance-transfer}
\begin{table}[t]
\fontencoding{OT1}\selectfont
\caption{\textbf{Low-budget reasoning and transfer.} All methods use the same low-budget pure diffusion setting and fast rollout. Acc. and Pass@1 are percentages; RTL is mean readable tokens, and Time is seconds per request. \tablekey{tableGroup}{Domain} separates mathematics and code; \tablekey{tableFocus}{HiLo-RD} identifies our method; \tablekey{tableBestPale}{Best} marks the highest task score or shortest response per benchmark. Bold and underlined values mark the best and second-best distinct values in each column, including ties.}
\label{tab:main_results_no_nfe}
\centering
\small
\setlength{\tabcolsep}{5pt}
\renewcommand{\arraystretch}{1.14}
\begin{tabularx}{\linewidth}{@{}l*{6}{>{\raggedleft\arraybackslash}X}@{}}
\toprule
& \multicolumn{3}{c}{GSM8K} & \multicolumn{3}{c}{MATH500}\\
\cmidrule(lr){2-4}\cmidrule(lr){5-7}
Method & Acc.$\uparrow$ & RTL$\downarrow$ & Time$\downarrow$ & Acc.$\uparrow$ & RTL$\downarrow$ & Time$\downarrow$\\
\midrule
\rowcolor{tableGroup}
\multicolumn{7}{@{}c@{}}{\bfseries In-domain mathematics reasoning}\\
vanilla & 68.84 & 362.8 & 73.35 & 28.40 & 479.6 & 31.09\\
Masked SFT & 70.13 & 366.8 & 72.10 & \underline{30.60} & \underline{478.3} & 31.03\\
Diffu-GRPO & \underline{72.93} & \underline{325.4} & \textbf{63.53} & 29.20 & 482.1 & 31.19\\
OPSD & 69.90 & 360.7 & 73.25 & 30.20 & 478.5 & \underline{30.96}\\
dOPSD & 70.05 & 363.9 & 73.47 & 29.40 & 478.4 & \underline{30.96}\\
\rowcolor{tableFocus}
\textbf{HiLo-RD} & \tablebest{78.24} & \tablebest{271.3} & \underline{71.71} & \tablebest{32.40} & \tablebest{458.2} & \textbf{29.90}\\
\bottomrule
\end{tabularx}
\vspace{6pt}
\begin{tabularx}{\linewidth}{@{}l*{6}{>{\raggedleft\arraybackslash}X}@{}}
\toprule
& \multicolumn{3}{c}{HumanEval} & \multicolumn{3}{c}{MBPP}\\
\cmidrule(lr){2-4}\cmidrule(lr){5-7}
Method & Pass@1$\uparrow$ & RTL$\downarrow$ & Time$\downarrow$ & Pass@1$\uparrow$ & RTL$\downarrow$ & Time$\downarrow$\\
\midrule
\rowcolor{tableGroup}
\multicolumn{7}{@{}c@{}}{\bfseries Out-of-domain code generation}\\
vanilla & 45.73 & 510.0 & 35.59 & \underline{38.60} & 151.0 & 19.51\\
Masked SFT & 47.56 & 510.0 & 35.53 & 37.40 & 148.9 & 19.36\\
Diffu-GRPO & 48.78 & 510.0 & 35.54 & 38.40 & \underline{127.1} & \underline{16.61}\\
OPSD & 46.95 & \underline{508.4} & \underline{35.52} & 38.00 & 146.9 & 19.11\\
dOPSD & \tablebest{50.00} & 510.0 & 35.64 & 37.80 & 143.1 & 18.73\\
\rowcolor{tableFocus}
\textbf{HiLo-RD} & \underline{49.39} & \tablebest{508.2} & \textbf{35.45} & \tablebest{41.40} & \tablebest{84.7} & \textbf{11.20}\\
\bottomrule
\end{tabularx}
\end{table}

The results and comparisons of the model performance using different methods under the same low-budget pure diffusion setting with fast rollout are shown in Table~\ref{tab:main_results_no_nfe}. \method\ improves over vanilla on all four benchmarks while producing shorter readable responses.

\paragraph{In-domain mathematical reasoning.}
We observe that \method\ improves low-budget mathematical reasoning across datasets of different difficulty, achieving the highest accuracy and shortest readable responses among the compared methods on both GSM8K and MATH500. The joint improvement matters because it shows that better reasoning need not come from extending the readable solution: higher-budget representations can instead help the same model use its smaller window more effectively. Compared with the other post-training objectives, representation self-distillation yields a more favorable balance between correctness and response length. This supports the motivation developed in Section~\ref{sec:panic-setup}: useful computation exposed by a larger budget can guide low-budget token predictions and commitment decisions, rather than requiring the larger budget at deployment. These gains align with the mechanism of Budget-Induced Premature Commitment: stronger low-budget reasoning no longer requires longer readable responses, although the attainable degree of compression remains task-dependent.

\paragraph{Out-of-domain code generation.}
The benefits extend to code generation without further code-domain training, suggesting that the learned improvement is not confined to mathematics-specific answer patterns. On MBPP, \method\ combines the strongest pass@$k$ with the shortest responses. On HumanEval, it remains competitive with dOPSD, which has a slight advantage in pass@$k$, while the reduction in readable length is marginal. This contrast shows that quality transfer need not be accompanied by equally strong compression across domains. Nevertheless, improvement over vanilla on both benchmarks supports the transferability of the decoding behavior learned through mathematics-based representation distillation. The same model's higher-budget computation thus supplies useful supervision beyond the training domain, without teacher-only reference solutions.

\subsection{Efficient Trajectory Acquisition (RQ3)}
\label{sec:rollout-efficiency}

\begin{wrapfigure}[21]{r}{0.54\textwidth}
\centering
\includegraphics[width=\linewidth]{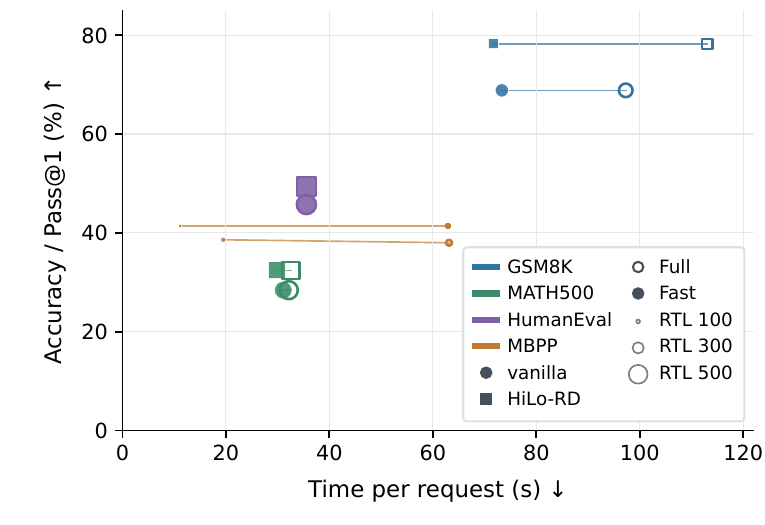}
\caption{\textbf{Quality--latency trade-off.} Color denotes dataset; shape and fill identify model and rollout. Linear marker size represents RTL. Lines pair Full/Fast results. Timings await normalization.}
\label{fig:rollout-tradeoff}
\end{wrapfigure}

Efficient trajectory acquisition reduces unnecessary decoding after response completion. Figure~\ref{fig:rollout-tradeoff} compares full and fast rollout for vanilla and \method; within each dataset, movement toward the upper left indicates higher quality and lower latency, while smaller linear marker size indicates shorter RTL. Fast rollout reduces the recorded latency most clearly on GSM8K and MBPP, with little change on HumanEval. In particular, the GSM8K \method\ pair has identical accuracy and RTL under full and fast rollout. This illustrates the distinction between the two efficiency gains: distillation improves what the model generates, whereas completion-aware stopping can save computation without shortening the completed response. Reported timing differences remain subject to cross-server normalization.

\paragraph{Why savings vary across tasks.}
The task-dependent pattern is consistent with the cost reduction from $T F(P+L)$ to $R F(P+L)$: early completion leaves more updates to skip, and each skipped update saves a forward pass over the full prompt and response window. Short MBPP responses leave substantial unnecessary continuation, whereas HumanEval responses nearly fill the window. Longer prompts, such as those used for GSM8K, also increase the saving per skipped update when other settings are fixed. RTL alone does not determine the stopping step, since commits outside the eventual readable prefix also consume updates. Thus, the observed efficiency depends jointly on response completion, commitment order, and input length. Under matched decoding and extraction, the stopping rule preserves the completed prefix (Proposition~\ref{prop:stop-prefix}); applied to teacher acquisition, it similarly avoids continuation after the student-sized supervision prefix is complete. This makes efficient acquisition complementary to representation distillation, rather than a substitute for learning better low-budget decoding.

\subsection{Toward Self-Evolving Quality--Budget Refinement (RQ4)}
\label{sec:iterative-refinement}

We investigate whether \method\ can bootstrap its own supervision across successive distillation cycles. Starting from the unadapted model C0, each cycle uses the preceding model to acquire new high-budget trajectories and distills them back into low-budget decoding. All cycles reuse the same fixed subset of 1,000 GSM8K training examples, so improvements arise from recursively refining the model's own cross-budget supervision rather than introducing additional data. Figure~\ref{fig:iterative-c0-c3} reports C0--C3, with exact measurements provided in Appendix~\ref{app:iterative-refinement}.

\begin{figure}[!ht]
\centering
\includegraphics[width=\linewidth]{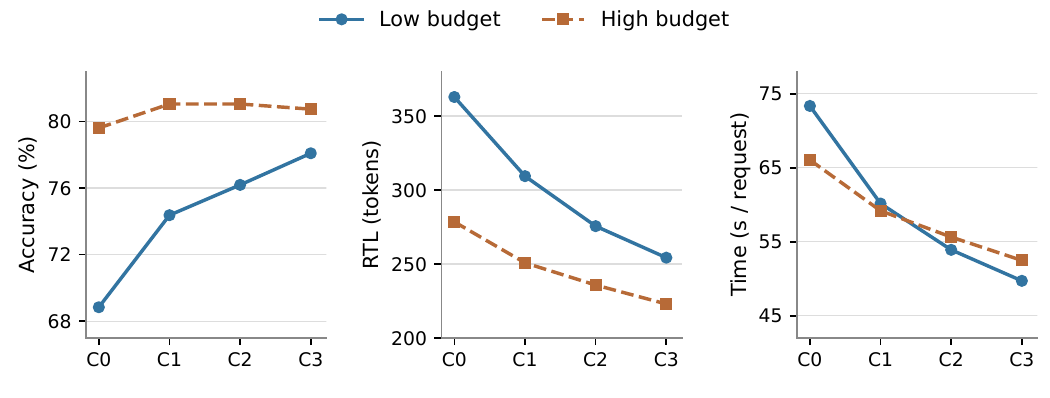}
\caption{\textbf{Successive cross-budget self-distillation on GSM8K.} Each cycle reuses the same 1,000 training examples and acquires new high-budget supervision from the preceding model. Across C0--C3, low-budget accuracy rises while readable length and latency fall; high-budget decoding remains stronger than C0 and becomes progressively more efficient.}
\label{fig:iterative-c0-c3}
\end{figure}

Successive distillation progressively improves the quality--budget frontier. From C0 to C3, low-budget accuracy rises while readable length and latency fall; high-budget decoding remains above C0 accuracy and becomes progressively shorter and faster. Because every cycle reuses the same data, these coupled gains arise from recursively improving the model's own supervision. This provides a concrete path toward self-evolving budget adaptation without an external teacher or additional training examples.

\section{Conclusion and Future Work}
We identify Budget-Induced Premature Commitment in masked diffusion language models: under a low yet sufficient response budget, the model can fix response-tail fragments before the reasoning that supports them is resolved. Decoding traces, attention diagnostics, and commitment-order interventions trace how this early decision redirects subsequent bidirectional reasoning and degrades the resulting trajectory. Guided by this mechanism, \method\ distills representations from the same model's reliable high-budget trajectories into low-budget decoding. The resulting model reasons more accurately with shorter responses across mathematical reasoning and code generation, realizing \emph{Think Wide, Reason Lean}; completed-prefix stopping further avoids unnecessary decoding work.

Successive distillation further improves the quality--budget frontier using the model's own trajectories, opening a path toward self-evolving diffusion reasoning. Incorporating richer intermediate supervision may help sustain this refinement across longer sequences of self-distillation cycles.

\newpage
\subsection*{AI use statement}
In this work, we used generative AI tools solely to polish the manuscript. We
 did not use generative AI tools to generate or analyze experimental results,
produce research artifacts, or make scientific claims. At least two authors
independently conducted a detailed review of every AI-assisted passage, checking
its accuracy, clarity, and consistency with the reported methods, experiments,
and claims. We take responsibility for the final content of this work, including
text, claims, or artifacts produced with the aid of generative AI.

\bibliography{iclr2027_conference}
\bibliographystyle{iclr2027_conference}
\clearpage
\appendix

\section{Budget-Induced Premature Commitment: Diagnostic Protocols and Additional Evidence}
\label{app:panic-diagnostics}

\subsection{Main diagnostic cohort and controlled comparisons}
The three diagnostics in Figures~\ref{fig:panic-budget-sensitivity}--\ref{fig:panic-decoding-intervention} use 155 randomly sampled cases exhibiting Budget-Induced Premature Commitment during low-budget decoding. These paired cases have an incorrect native $L=T=512$ response and a correct native $L=T=1024$ response. All three diagnostics use this same cohort; recovery rates in Figure~\ref{fig:panic-decoding-intervention} are measured within these failure/success pairs. Figure~\ref{fig:panic-distribution} reports the population-level length and accuracy patterns. The complementary pilot and illustrative traces below use separate cohorts.

\paragraph{Window expansion before commitment.}
\label{app:panic-window-construction}
At the first ordinary response-tail commitment in each low-budget trace, we evaluate the immediately preceding state using the query at the selected tail position. The extended view appends 512 visible MASK positions while retaining the prompt, original token identities, position indices, and query. Native high-budget decoding is replayed to the same per-case update index and uses its own next-selected position as the query. The extension therefore tests the same state's sensitivity to visible response space, while the native comparison characterizes the computation reached along a high-budget trajectory. Attention is averaged over heads within each case and then across cases, using 16-position generation-key bins, all 32 layers, and a shared linear scale. The highlighted interval is response positions 500--520. Values are raw attention percentages over each condition's full key set. Shared-key normalization and fixed-target confidence probes below separate changes in attention normalization from changes in commitment preference.

\paragraph{Response tail and replay state.}
Let $e$ be the realized response extent, ending at the first spatial EOT, otherwise the first EOS, or the last ordinary token if neither is present. With zero-based response indices, the tail set is
\begin{equation}
 \mathcal E=\{e-\lceil 0.1e\rceil,\ldots,e-1\},\qquad
 t_\star=\min\{t:j_t\in\mathcal E,\ v_t\notin\{\mask,\mathrm{EOS},\mathrm{EOT}\}\}.
\end{equation}
Here $(j_t,v_t)$ is the position and token committed at update $t$. This tail follows the realized response; the allocated-window boundary is at $L$. They lie close together in boundary-filling low-budget traces. Replay constructs $\mathbf{x}_\star$ using only commitments with $t<t_\star$, leaving the event position masked. The original tail set remains fixed in subsequent probes and interventions.

\begin{algorithm}[H]
\caption{Same-state visible-window expansion}
\label{alg:panic-window-expansion}
\small
\begin{algorithmic}[1]
\Require Model $f_\theta$, prompt $\mathbf{q}$, low-budget trace $\tau$, lengths $L_\ell<L_h$
\Ensure Attention and prediction measurements before the next commitment
\State Identify tail $E$ and first ordinary tail event $(t_\star,j_\star,v_\star)$ in $\tau$
\State $\mathbf{x}_\star\gets[\mathbf{q};\mask^{L_\ell}]$
\For{each commitment $(t,j_t,v_t)$ in $\tau$ with $t<t_\star$}
    \State Set response position $j_t$ in $\mathbf{x}_\star$ to $v_t$
\EndFor
\State $\mathbf{x}^+\gets[\mathbf{x}_\star;\mask^{L_h-L_\ell}]$
\State $\mathcal D_\ell\gets f_\theta(\mathbf{x}_\star)$ \Comment{Original window}
\State $\mathcal D_+\gets f_\theta(\mathbf{x}^+)$ \Comment{Visible extension}
\State $\mathcal D_0\gets f_\theta(\mathbf{x}^+;\text{appended keys hidden})$ \Comment{Padding control}
\State Exclude appended positions from selection in $\mathcal D_0$
\State Fix the front candidate using $\mathcal D_\ell$
\For{$\mathcal D\in\{\mathcal D_\ell,\mathcal D_+,\mathcal D_0\}$}
    \State Record attention at query $j_\star$ and probabilities of the fixed targets
    \State Record commitment scores and the next eligible position
    \State Compare scores restricted to the original masked positions
\EndFor
\State \Return Recorded measurements
\end{algorithmic}
\end{algorithm}

Appended visible masks participate in bidirectional attention and are eligible for the diagnostic next-position selection. The additional shared-position comparison restricts this competition to the original window. These are single-forward probes; native high-budget trajectories are decoded separately from initialization.

\paragraph{Attention after commitment.}
The replay retains fixed front queries selected from unfilled positions and fixed tail/prefix key sets at Before, After, and offsets $+2,+4,+8,+16,+32$. Each value is a paired change from Before, averaged over heads and queries within a case and then over valid cases. Coverage is 145/155 cases. Immediately after commitment, tail attention increases by 1.908 percentage points on average, with increases in 140/145 cases; prefix attention changes by $-0.413$ points. At $+32$, the respective changes are $+4.194$ and $-1.875$ points. Later offsets describe the evolving trajectory, including additional natural commitments; the immediate Before/After pair measures the local change around the first tail commitment.

\paragraph{Interventions on commitment order.}
\label{app:panic-order-construction}
Low Budget retains the archived low-budget schedule. Tail Delay prevents selection of the original response-tail positions for the first 10\% of the remaining low-budget updates. Matched-order Control changes decisions at the same intervention offsets by excluding its own current highest-scoring position, without a tail-specific rule. Both branches start from $\mathbf{x}_\star$ and retain $L=T=512$, one commitment per update, and the full decoding schedule. Correct outcomes number 67 and 47, respectively. The paired difference is 12.903 percentage points, with a 95\% interval of [7.742,18.710] from 2,000 bootstrap resamples of questions. This tests recovery through commitment order at fixed budget.

\paragraph{Matched intervention schedule.}
The construction in Algorithm~\ref{alg:panic-order-interventions} sets $R=T-t_\star$ and $D=\max(1,\lceil 0.1R\rceil)$. Let $\operatorname{Top}(\mathbf{x},\mathcal A)$ return the highest-confidence currently masked position in eligible set $\mathcal A$, and let $\widehat v_j(\mathbf{x})$ be its predicted token. The same predictor and confidence rule are used in both branches. Excluding a position changes eligibility for that update; it leaves all committed tokens intact.

\begin{algorithm}[H]
\caption{Tail Delay and Matched-order Control}
\label{alg:panic-order-interventions}
\small
\begin{algorithmic}[1]
\Require Pre-event state $\mathbf{x}_\star$, original tail $\mathcal E$, remaining updates $R$, delay $D$
\Ensure Tail-delay continuation $\mathbf{x}^{\mathrm{delay}}$ and matched-order continuation $\mathbf{x}^{\mathrm{control}}$
\State $\mathbf{x}\gets\mathbf{x}_\star$; $\mathcal K\gets\varnothing$
\For{$k=0,\ldots,R-1$} \Comment{Tail Delay}
    \State $\mathcal M\gets\{j:x_j=\mask\}$; $i\gets\operatorname{Top}(\mathbf{x},\mathcal M)$; $j\gets i$
    \If{$k<D$ \textbf{and} $\mathcal M\setminus\mathcal E\ne\varnothing$}
        \State $j\gets\operatorname{Top}(\mathbf{x},\mathcal M\setminus\mathcal E)$
    \EndIf
    \If{$j\ne i$}
        \State $\mathcal K\gets\mathcal K\cup\{k\}$ \Comment{Record changed decisions}
    \EndIf
    \State $x_j\gets\widehat v_j(\mathbf{x})$
\EndFor
\State $\mathbf{x}^{\mathrm{delay}}\gets\mathbf{x}$; $\mathbf{x}\gets\mathbf{x}_\star$
\For{$k=0,\ldots,R-1$} \Comment{Matched-order Control}
    \State $\mathcal M\gets\{j:x_j=\mask\}$; $i\gets\operatorname{Top}(\mathbf{x},\mathcal M)$; $j\gets i$
    \If{$k\in\mathcal K$ \textbf{and} $\mathcal M\setminus\{i\}\ne\varnothing$}
        \State $j\gets\operatorname{Top}(\mathbf{x},\mathcal M\setminus\{i\})$
    \EndIf
    \State $x_j\gets\widehat v_j(\mathbf{x})$
\EndFor
\State $\mathbf{x}^{\mathrm{control}}\gets\mathbf{x}$
\State \Return $\mathbf{x}^{\mathrm{delay}},\mathbf{x}^{\mathrm{control}}$
\end{algorithmic}
\end{algorithm}

The control's excluded position is recomputed from its own state and can itself lie in the tail. Thus, the comparison is between a tail-defined eligibility constraint and a matched schedule of current-winner exclusions. When an exclusion would leave no candidate, the decoder uses its default winner and records the fallback. The reference decoder predicts tokens by argmax, suppresses EOS/EOT in the confidence calculation, masks already committed positions, and then chooses one eligible position. These order interventions operate within the original window throughout.

\subsection{Complementary checks on visible space and tail reading}
\label{app:panic-pilot}
\paragraph{Visible masks versus invisible padding.}
This check uses 20 low-budget failure/high-budget short-success pairs. At the state immediately before the first ordinary tail commitment, we retain the same prompt, committed contents, event-tail query, and token targets across the original window, visible-mask expansion, and invisible padding. Visible expansion reduces mean tail-target probability from 46.30\% to 0.25\%, while invisible padding gives 46.24\%. The fixed front target changes from 47.65\% to 52.12\% with visible expansion and to 47.79\% with invisible padding. The next selected position moves away from the original tail in 19/20 cases, and the confidence margin reverses even when competition is restricted to original positions. These comparisons isolate the effect of visible generation space from added tensor length and competition from new positions. Algorithm~\ref{alg:panic-window-expansion} specifies the construction.

\paragraph{Temporary restriction of tail reading.}
A complementary intervention blocks front queries from attending to committed ordinary tail keys in every layer for an initial interval. It repairs one case and breaks one among the 20 pairs, leaving aggregate accuracy unchanged. This intervention restricts access to already committed content; the order interventions in Figure~\ref{fig:panic-decoding-intervention} instead postpone its commitment while retaining normal bidirectional attention.

\subsection{Irreversible commitment and reachable completions}
\label{app:panic-measurement}
For a partial response, let
\begin{equation}
 \mathcal F_t(L)=\{\mathbf{y}\in\mathcal V^L:y_j=y_{t,j}
     \text{ for every }j\in C_t\}.
 \qquad \mathcal F_{t+1}(L)\subseteq\mathcal F_t(L).
 \label{eq:compatible-completions}
\end{equation}
If $\mathcal G(\mathbf{q},L)$ denotes correct completions under a specified evaluator and extraction rule, then $\mathcal F_{t+1}(L)\cap\mathcal G(\mathbf{q},L)\subseteq\mathcal F_t(L)\cap\mathcal G(\mathbf{q},L)$. A wrong fixed fragment can remove compatible correct continuations; its content also changes subsequent predictions and position scores. Since both correct and incorrect commitments shrink the set, the relevant question is which successful continuations remain reachable. The order interventions test whether a different sequence of commitments can recover successful paths at the same budget. Budget-Induced Premature Commitment denotes premature boundary-near commitment while the intervening reasoning remains unresolved despite sufficient response space.

\subsection{Paired trace examples and visualization protocol}
\label{app:panic}
\label{app:panic-additional}
The paired traces in Figures~\ref{fig:panic-trace-cases-1} and~\ref{fig:panic-trace-cases-2} show 16 GSM8K responses at $L=T=512$ and $L=T=1024$. The supplied semantic-answer audit verifies an incorrect low-budget answer and a correct high-budget answer in each pair. Selection additionally requires the low-budget answer to be fully committed within the first 35\% of response decoding, the high-budget answer in the final 10\%, and a tail-before-prefix inversion score of at least 0.85 and at most 0.08, respectively. These examples illustrate an outcome and ordering contrast, not the prevalence of Budget-Induced Premature Commitment.

\paragraph{Reading the figures.}
Tokens follow final response order. Color encodes $(s+1)/R_{\mathrm{resp}}$, where $s$ is the stored zero-based commitment update and $R_{\mathrm{resp}}$ is the exported response-completion step. Purple is early, yellow is late, and gray marks special tokens. Normalization is separate for each response: equal colors do not imply equal update counts or equal fractions of the configured budget. The ribbon covers the realized response; text excerpts show its beginning and end. Headers report readable token length (RTL), while panel widths are rescaled for layout. RTL is 510 in every selected low-budget case and ranges from 162 to 355 at the high budget; neither RTL nor ribbon width is a count of model calls.

The low-budget ribbons expose early answer commitments followed by middle infilling. The paired high-budget responses commit their answers near completion and are shorter. For example, \texttt{gsm8k\_test\_0314} fixes the wrong answer 20 early at the low budget but reaches 30 near high-budget completion. The audited answer, rather than an automatically extracted number, is displayed. Extraction false negatives such as \texttt{gsm8k\_test\_0527} are excluded. In \texttt{gsm8k\_test\_1248}, the reviewed low-budget answer is 251 whereas the archived grader extracts 15; both are incorrect relative to 750. This audit prevents grading artifacts from being presented as reasoning failures without retroactively changing benchmark scores.

\begin{figure}[p]
\centering
\includegraphics[width=\linewidth,height=0.86\textheight,keepaspectratio,trim=3mm 10mm 0mm 10mm,clip]{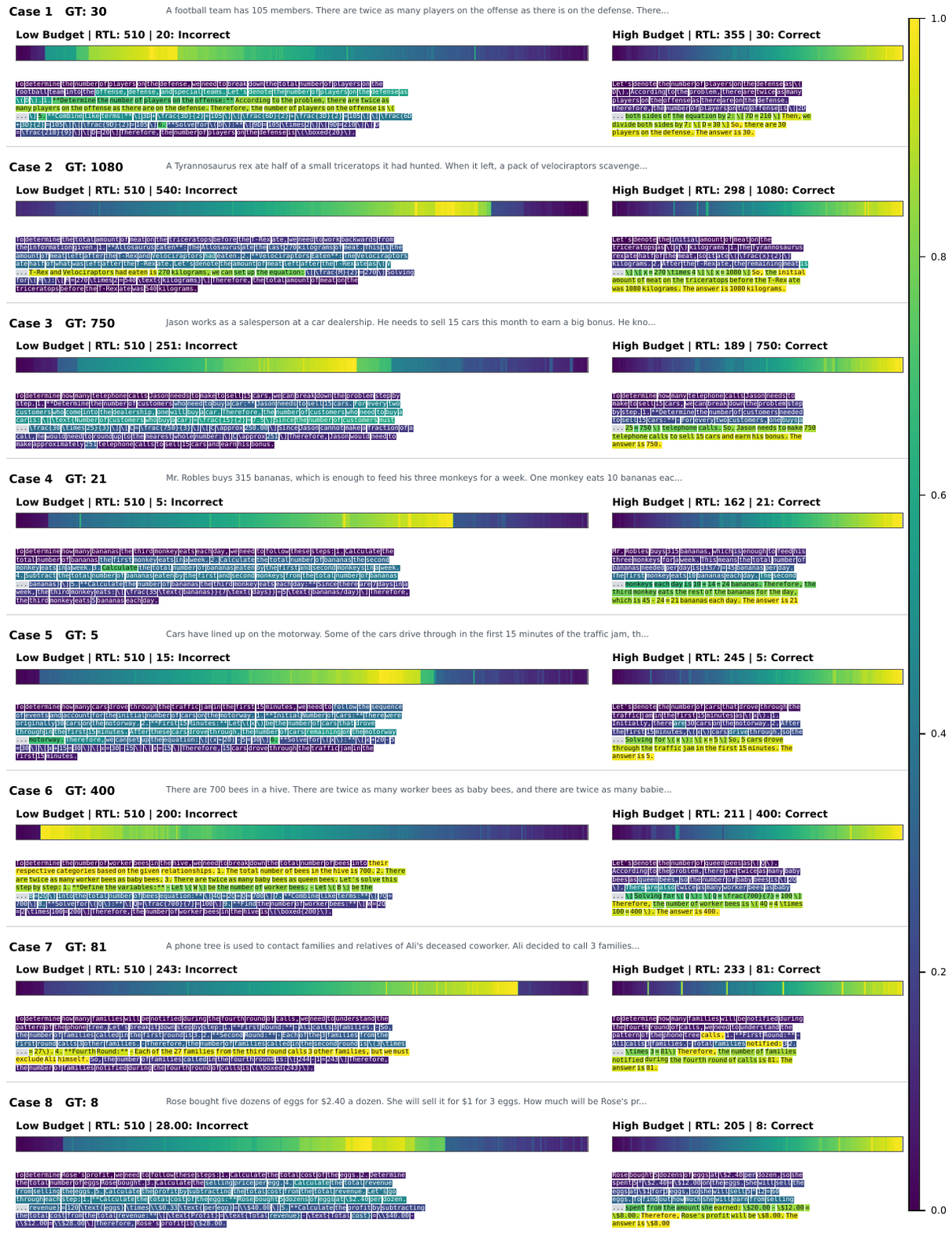}
\caption{\textbf{Paired premature-commitment traces, Cases 1--8.} Left: low budget; right: high budget. Headers give RTL. Ribbons and excerpts encode commitment timing in final response order: purple is early, yellow late, and gray denotes special tokens. Timing is normalized per response, and panel widths are rescaled. Cases are selected for contrasting outcomes and commitment order.}
\label{fig:panic-trace-cases-1}
\end{figure}

\begin{figure}[p]
\centering
\includegraphics[width=\linewidth,height=0.86\textheight,keepaspectratio,trim=3mm 10mm 0mm 10mm,clip]{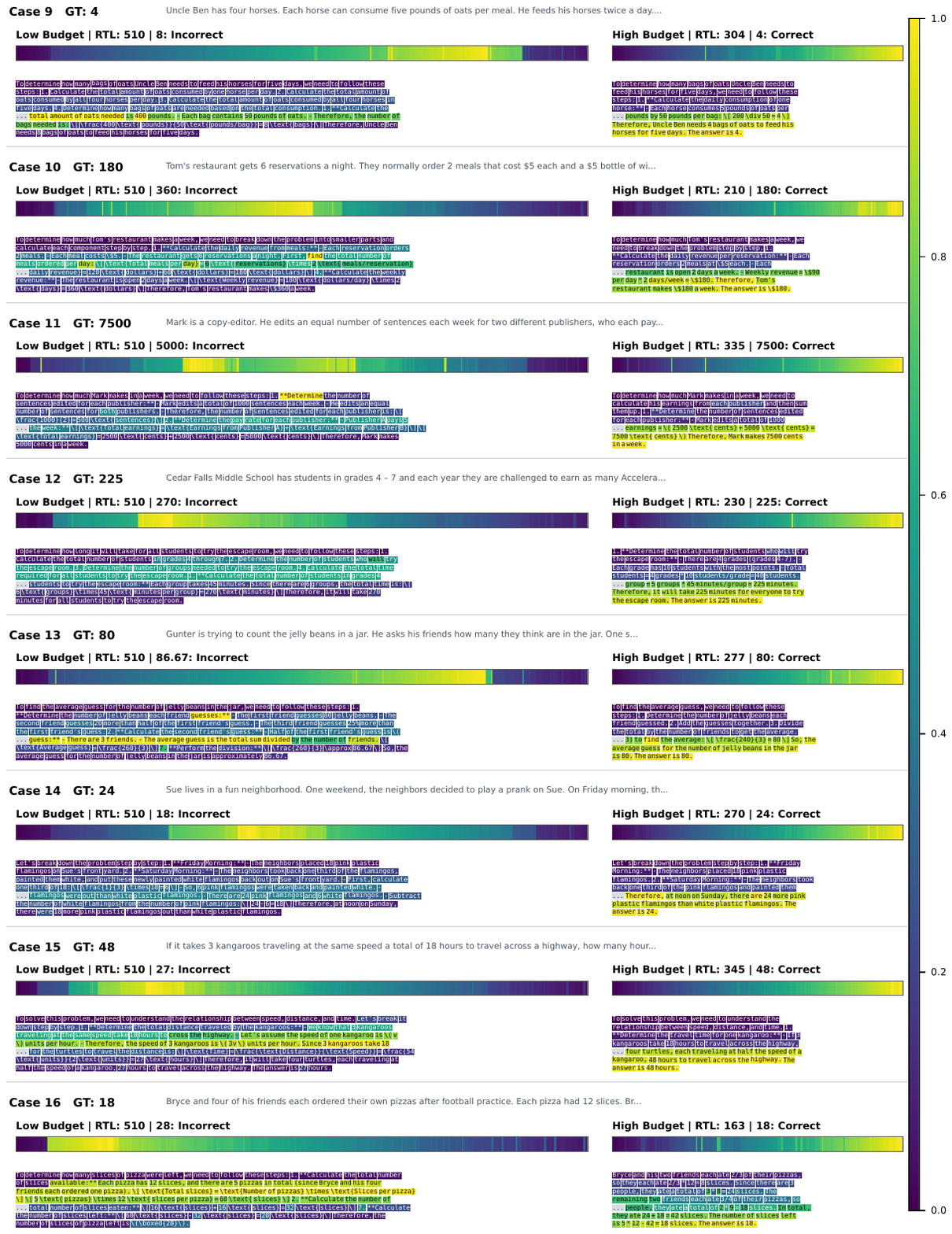}
\caption{\textbf{Paired premature-commitment traces, Cases 9--16.} The encoding follows Figure~\ref{fig:panic-trace-cases-1}. Low-budget incorrect answers are fixed before much of the intervening reasoning; the paired high-budget answers are correct and committed near completion. RTL provides the absolute length comparison.}
\label{fig:panic-trace-cases-2}
\end{figure}

\clearpage
\section{HiLo-RD: Objective Details and Local Guarantees}
\label{app:method-details}

\subsection{Trajectory sampling and objective normalization}
\label{app:objective-details}
Within a training phase, the frozen teacher generates
\begin{equation}
 \xi=(\mathbf{y}_0,\ldots,\mathbf{y}_R)
 \sim\mathsf P_{\bar\theta}^{B_{\mathrm T}}(\cdot\mid\mathbf{q}),
 \qquad R\leq T_{\mathrm T}.
 \label{eq:teacher-trajectory}
\end{equation}
The reference fast-acquisition implementation retains at most $K=8$ states, including the initial all-MASK state. Each paired forward uses the pre-update state $\mathbf{y}_t$: the teacher processes the full window and the student processes its crop. Supervision eligibility is determined after the corresponding update:
\begin{equation}
 m_{t,j}=\mathbf{1}\{j\in C_{t+1}\},\qquad
 n_t=\sum_{j=1}^{L_{\mathrm S}}m_{t,j}.
 \label{eq:supervision-mask}
\end{equation}
This cumulative mask includes previously committed positions and positions newly committed in the update. A newly selected position can therefore be supervised while its input is still masked. Earlier commitments receive repeated supervision as surrounding content develops. The loss in Eq.~(\ref{eq:state-loss}) normalizes by $d n_t$, so a later state does not receive more weight solely because more positions have been filled.

Let $\mathcal I_+(\xi)$ contain the retained states with $n_t>0$. We average their losses:
\begin{equation}
 \ell_\xi(\theta)=\frac{1}{|\mathcal I_+(\xi)|}
 \sum_{t\in\mathcal I_+(\xi)}\ell_t(\theta;\xi),
 \label{eq:trajectory-loss}
\end{equation}
with $\ell_\xi=0$ if the set is empty. For a batch $\mathcal B$ and a trajectory gate $g(\xi,\mathbf{q},a)\in\{0,1\}$,
\begin{equation}
 \mathcal L_{\mathrm{HiLo\text{-}RD}}(\theta;\mathcal B)=
 \frac{\sum_{(\mathbf{q},a,\xi)\in\mathcal B}
       g(\xi,\mathbf{q},a)\ell_\xi(\theta)}
 {\max\{1,\sum_{(\mathbf{q},a,\xi)\in\mathcal B}
       g(\xi,\mathbf{q},a)\}}.
 \label{eq:batch-loss}
\end{equation}
The reported positive training recipe accepts teacher responses whose extracted answers pass verification. An empty accepted batch contributes zero signal. The answer label selects rollouts; it is not supplied as teacher input or used as the representation target. This filter certifies the extracted answer, not every intermediate reasoning step. Setting $g\equiv1$ yields an unfiltered variant without correctness labels. No gradient passes through the sampled tokens, commitment decisions, or frozen teacher.

\subsection{Completed-prefix stopping}
\label{app:cost-stopping}
The clean-EOT rule and the prefix-cap rule in Section~\ref{sec:fast-rollout} share a preservation property under irreversible decoding.

\begin{proposition}[Completed-prefix preservation]
\label{prop:stop-prefix}
Consider full and fast decoders with the same initialization, update schedule, tie-breaking, and coupled randomness. Assume that committed tokens are never revised. If fast decoding stops after $R$ updates, its states and actions through $R$ equal the corresponding prefix of full decoding. Every fully committed response prefix is unchanged by continuing the full decoder. If the leftmost committed EOT is at position $e$ and all preceding positions are committed, extraction strictly before the first EOT returns the same token sequence after either run.
\end{proposition}
\noindent\emph{Proof.}
Before termination, both decoders apply identical updates to identical states. Induction gives trajectory-prefix equality. Irreversibility then preserves every token in the completed prefix. Under clean EOT, all positions through $e$ are fixed, so no earlier EOT can subsequently appear and extraction before $e$ is unchanged.\hfill$\square$

Stopping at a completed cap $c<L$ preserves only the first $c$ positions and may truncate a solution that continues beyond them. The teacher can use $c=L_{\mathrm S}$ while retaining full-window attention and allowing commitments outside that prefix before termination. The evaluation driver's default cap instead equals its allocated response length. Both rules must execute inside the decoding loop; discarding later states after a complete rollout does not save forward passes.

The output guarantee requires the stated extraction rule. The archived evaluation driver decodes the returned response array with \texttt{skip\_special\_tokens=True}, without explicitly cutting at the first EOT. Already committed post-EOT text can then enter grading. Prefix preservation alone does not establish identical scores for that extractor; raw outputs, extracted answers, and stop reasons must be compared under a common protocol.

\subsection{Representation matching and token distributions}
\label{app:proofs}
The vocabulary head translates a representation difference into a logit difference. The following bound formalizes this local relation without assuming identical autonomous decoding paths.

\begin{proposition}[Local token-distribution control]
\label{prop:token-control}
Let $V=|\mathcal V|$ and let $z^{\mathrm T}=Wh^{\mathrm T}+b$ and $z^{\mathrm S}=Wh^{\mathrm S}+b$ be finite logits under the same fixed head $W\in\mathbb R^{V\times d}$. For $\tau>0$, define $p^{\mathrm T}_\tau=\softmax(z^{\mathrm T}/\tau)$, $p^{\mathrm S}_\tau=\softmax(z^{\mathrm S}/\tau)$, $\mathsf C=I_V-\mathbf{1}\mathbf{1}^\top/V$, and $e=\|h^{\mathrm S}-h^{\mathrm T}\|_2^2/d$. Then
\begin{equation}
 D_{\mathrm{KL}}(p^{\mathrm T}_\tau\|p^{\mathrm S}_\tau)
 \leq\frac{d\|\mathsf CW\|_{\mathrm{op}}^2}{4\tau^2}\,e.
 \label{eq:kl-bound}
\end{equation}
The same bound holds with the KL arguments exchanged.
\end{proposition}

\noindent\emph{Proof.}
Let $A(u)=\log\sum_{v=1}^{V}e^{u_v}$. Its gradient is softmax, and its Hessian is $\operatorname{diag}(p)-pp^\top$, where $p=\softmax(u)$. For any unit vector $w$, the Hessian quadratic form is the variance of $w_i$ under $i\sim p$, so
\begin{equation}
 w^\top\nabla^2A(u)w
 \leq\frac{(\max_i w_i-\min_i w_i)^2}{4}
 \leq\frac12.
\end{equation}
Thus, $A$ has a $1/2$-Lipschitz gradient and
\begin{equation}
 A(v)-A(u)-\nabla A(u)^\top(v-u)\leq\tfrac14\|v-u\|_2^2.
 \label{eq:bregman-bound}
\end{equation}
Set $u=\mathsf Cz^{\mathrm T}/\tau$ and $v=\mathsf Cz^{\mathrm S}/\tau$. Centering leaves softmax unchanged. Substitution gives
\begin{align}
 D_{\mathrm{KL}}(p^{\mathrm T}_\tau\|p^{\mathrm S}_\tau)
 &=A(v)-A(u)-\nabla A(u)^\top(v-u)\\
 &\leq\frac{\|\mathsf C(z^{\mathrm S}-z^{\mathrm T})\|_2^2}{4\tau^2}\\
 &=\frac{\|\mathsf CW(h^{\mathrm S}-h^{\mathrm T})\|_2^2}{4\tau^2}
 \leq\frac{d\|\mathsf CW\|_{\mathrm{op}}^2}{4\tau^2}e.
\end{align}
Exchanging $u$ and $v$ proves the reverse-direction bound; the common bias cancels.\hfill$\square$

Applying the objective's nonnegative averaging weights gives the corresponding average bound on supervised pairs. The result concerns finite raw logits and does not cover hard logit suppression or discrete selection directly. If the teacher's unique top logit has margin $\gamma>0$, the sufficient condition $\|\mathsf CW\|_{\mathrm{op}}\|h^{\mathrm S}-h^{\mathrm T}\|_2<\gamma/2$ also preserves its token argmax: every centered logit changes by less than half the margin. Neither statement guarantees position-selection agreement, because other competing positions may not be supervised and autonomous student states may differ.

\subsection{What the smaller-window representation can retain}
\label{sec:interpretation}
Fix a teacher phase and let $\mu$ be the distribution over supervised position pairs induced by the rollout, filter, state sampling, and normalization in Eq.~(\ref{eq:batch-loss}). Write $U=(\mathbf{q},\mathbf{y}_{t,1:L_{\mathrm S}},j)$ for the student-visible state and queried position, and $Z=h^{\mathrm T}_{t,j}$ for the target; the student does not receive $t$ as a separate input. For square-integrable $Z$ and any square-integrable predictor $f(U)$,
\begin{align}
 \E_\mu\|f(U)-Z\|_2^2
 &=\E_\mu\|f(U)-\E_\mu[Z\mid U]\|_2^2\nonumber\\
 &\quad+\E_\mu\|Z-\E_\mu[Z\mid U]\|_2^2.
 \label{eq:regression-decomposition}
\end{align}
To see this, expand around $m(U)=\E_\mu[Z\mid U]$; the cross term has zero conditional expectation given $U$. Hence the unrestricted optimum is $f^*(U)=\E_\mu[Z\mid U]$. The objective transfers the component of higher-budget representations predictable from the shared prompt and cropped state. It need not reconstruct every detail that depends on omitted response positions. A parameterized student may incur approximation and optimization error, and this regression identity does not guarantee reasoning quality on its autonomous rollouts.

\subsection{Optional outcome-conditioned repulsion}
\label{app:outcome-variants}
These variants supplement the positive objective and are considered separately, with the other repulsion branch disabled. Repulsion is disabled in the reference configuration. Let $D_{\mathrm H}$ be the higher-budget trajectory distance in Eq.~(\ref{eq:trajectory-loss}), and let $g_{\mathrm H}$ indicate final-answer correctness. A failed-high-budget variant uses
\begin{equation}
 \mathcal L_{\mathrm{failed\text{-}H}}
 =\operatorname{Avg}_{g_{\mathrm H}=1}D_{\mathrm H}
 -\lambda_{\mathrm H}\operatorname{Avg}_{g_{\mathrm H}=0}\min(D_{\mathrm H},m_{\mathrm H}).
\end{equation}
The two groups are normalized separately, and empty groups contribute zero. Because $-\min(D,m)=[m-D]_+-m$, the negative term is gradient-equivalent to margin repulsion, which stops increasing separation beyond $m_{\mathrm H}$.

A separate unfiltered-positive variant draws an additional failed rollout with fewer updates:
\begin{equation}
 \mathcal L_{\mathrm{failed\text{-}L}}
 =\operatorname{Avg}_{\mathcal B}
 [D_{\mathrm H}+\lambda_{\mathrm L}(1-g_{\mathrm L})[m_{\mathrm L}-D_{\mathrm L}]_+].
\end{equation}
Here $D_{\mathrm L}$ uses that rollout's own paired states; its penalty is averaged over the whole batch rather than only the failures. In the supplied implementation, this negative teacher retains the long window and reduces the step count, e.g., $(L_{\mathrm{neg}},T_{\mathrm{neg}})=(1024,128)$, not $(128,128)$. Removing $1-g_{\mathrm L}$ disables failure-only gating. Applying the high-budget correctness gate instead gates both terms on the rank, producing a distinct recipe; the failed-high-budget branch and that gate use mutually exclusive code paths. Representation repulsion does not guarantee a lower probability for a specified erroneous trajectory.

\section{Implementation and Evaluation Details}
\label{app:implementation}

\subsection{Training configuration and deployment}
The supplied GSM8K reference configuration is recorded in Table~\ref{tab:training-config}. The main experiments train separate checkpoints on the full GSM8K and original MATH training sets and evaluate them on GSM8K and MATH500, respectively. Code evaluation uses mathematics-trained checkpoints without further code-domain adaptation. The configuration below describes the documented \method\ run; it is not a complete configuration record for every dataset and baseline.

\begin{table}[!ht]
\centering
\caption{\textbf{GSM8K reference training configuration.} \tablekey{tableBestPale}{Teacher} and \tablekey{tableLowPale}{Student} distinguish the two budget views. The remaining settings document trajectory selection and adapter training for the full-training-set run. The 1,000-example successive-distillation study uses a separate training subset.}
\label{tab:training-config}
\small
\setlength{\tabcolsep}{6pt}
\renewcommand{\arraystretch}{1.15}
\begin{tabular}{@{}ll@{}}
\toprule
Component & Setting\\
\midrule
Backbone & LLaDA-8B-Instruct\\
Teacher budget & \cellcolor{tableBestPale}$L_{\mathrm T}=T_{\mathrm T}=1024$\\
Student training window & \cellcolor{tableLowPale}$L_{\mathrm S}=512$\\
Teacher response-prefix cap & $c=L_{\mathrm S}$\\
Retained states & Up to $K=8$, including the initial state\\
Representation & Final RMSNorm, before vocabulary projection\\
Supervision positions & Post-update cumulative committed positions\\
Trajectory filter & Teacher final-answer correctness\\
Adapter & LoRA rank 16, scaling 32\\
Learning rate / warmup & $10^{-5}$ / 50 optimizer steps\\
Evaluated checkpoint & Optimizer step 900\\
Training pool & GSM8K train, 7,473 examples\\
Checkpoint interval & 100 optimizer steps\\
Training hardware & Eight NVIDIA A100 80GB GPUs\\
Auxiliary repulsion & Disabled\\
\bottomrule
\end{tabular}
\end{table}

Only the student's LoRA parameters are optimized; the teacher, backbone, and vocabulary head remain fixed. At deployment, the adapter can be merged into the backbone and the student decodes autonomously. The number of sampled training states does not determine its deployment step budget $T_{\mathrm S}$.

\subsection{Computation and latency}
Equation~(\ref{eq:forward-cost}) describes dense recomputation without classifier-free guidance. It does not assume that quadratic attention dominates at the studied lengths or exclude implementations with partial reuse. For variable-window decoding, the corresponding cost is $\sum_{t=0}^{R-1}F(P+L_t)$.

Measured latency also depends on hardware, kernels, batching, and prompt length. The reported Full/Fast timings include results from different servers and remain subject to normalization; their ratios are not controlled speedup estimates. Mean per-request latency is distinct from dataset throughput and total multi-GPU wall time. Teacher-acquisition savings require stopping the decoding loop, rather than retaining fewer states after a full rollout.

\subsection{Readable length, extraction, and pass@$k$}
\label{app:metrics}
The archived evaluation driver measures RTL by retokenizing the decoded response after removing special tokens. This can differ from the number of occupied response positions or the length extracted before the first EOT. RTL is neither the configured window size nor the realized number of forward passes: out-of-prefix commitments also consume updates. The prefix-preservation guarantee in Appendix~\ref{app:cost-stopping} therefore requires an explicit extraction convention when comparing Full and Fast outputs.

For a problem with $n$ independently sampled completions and $c$ passing all code tests, the standard estimator is
\begin{equation}
 \widehat{\mathrm{pass@}k}
 =1-\frac{\binom{n-c}{k}}{\binom nk},\qquad n\geq k.
\end{equation}
The dataset score averages this quantity over problems. Its interpretation depends on the sample count, decoding distribution, and code-test protocol; a per-completion latency is not the total cost of obtaining a multi-sample pass@$k$ score.

\section{Successive Distillation: Procedure and Measurements}
\label{app:iterative-refinement}
\label{app:iteration}
Each cycle merges the preceding adapter into the backbone before initializing the next frozen teacher and trainable student. If $\theta^{(r-1)}$ is the merged model, cycle $r$ uses $\bar\theta^{(r)}=\theta^{(r-1)}$ and a student initialized from the same weights with a new adapter $\Delta^{(r)}$. The resulting merged model supplies the next cycle's supervision. This extension changes the teacher between cycles; the teacher remains fixed within each cycle.

The measurements underlying Figure~\ref{fig:iterative-c0-c3} are provided in Table~\ref{tab:iterative-c0-c3}. Every cycle reuses the same 1,000 GSM8K training examples selected with a fixed seed. C0 is the unadapted model, and C1--C3 denote successive cycles; this sequence is separate from the full-training-set main comparison.

\begin{table}[!ht]
\fontencoding{OT1}\selectfont
\centering
\caption{\textbf{Successive distillation improves the low-budget quality--length trade-off.} C0--C3 reuse the same 1,000 GSM8K training examples. Accuracy is in percent, RTL is mean readable tokens, and Time is seconds per request. \tablekey{tableFocus}{C3} identifies the final cycle; \tablekey{tableBestPale}{Best} marks the highest accuracy or shortest response within each budget. Bold indicates the best value in each column, including ties.}
\label{tab:iterative-c0-c3}
\small
\setlength{\tabcolsep}{6pt}
\renewcommand{\arraystretch}{1.16}
\begin{tabularx}{\linewidth}{@{}l*{6}{>{\raggedleft\arraybackslash}X}@{}}
\toprule
& \multicolumn{3}{c}{\bfseries Low budget} & \multicolumn{3}{c}{\bfseries High budget}\\
\cmidrule(lr){2-4}\cmidrule(lr){5-7}
Cycle & Acc.$\uparrow$ & RTL$\downarrow$ & Time$\downarrow$ & Acc.$\uparrow$ & RTL$\downarrow$ & Time$\downarrow$\\
\midrule
C0 (vanilla) & 68.84 & 362.8 & 73.35 & 79.61 & 278.4 & 66.05\\
C1 & 74.37 & 309.3 & 60.16 & \tablebest{81.05} & 250.7 & 59.18\\
C2 & 76.19 & 275.5 & 53.90 & \tablebest{81.05} & 235.7 & 55.63\\
\rowcolor{tableFocus}
C3 & \tablebest{78.09} & \tablebest{254.2} & \textbf{49.71} & 80.74 & \tablebest{222.9} & \textbf{52.47}\\
\bottomrule
\end{tabularx}
\end{table}

Low-budget accuracy improves monotonically from C0 to C3, accompanied by consistent reductions in readable length and latency. High-budget decoding also becomes substantially more efficient while retaining an accuracy advantage over C0. The improvement at both budgets allows each adapted model to generate more effective supervision for the following cycle, supporting successive cross-budget distillation as a practical mechanism for self-evolving quality--budget refinement.

\end{document}